\documentclass[letterpaper, 10pt, conference]{ieeeconf}  

\IEEEoverridecommandlockouts                              

\usepackage{amsmath}
\usepackage{amssymb}

\usepackage{changes}
\usepackage{amsthm}
\usepackage{graphicx}
\usepackage{booktabs}
\usepackage{subcaption}
\usepackage[percent]{overpic}
\usepackage{mathtools}
\usepackage{array}
\usepackage{multicol}
\usepackage{siunitx}
\usepackage{lipsum}
\usepackage{cuted}
\usepackage{stfloats} 
\usepackage[colorlinks=true, allcolors=blue]{hyperref}

\newtheorem{proposition}{Proposition}

\usepackage{xcolor}
\usepackage{todonotes}

\begin{document}

\title{\LARGE \bf
    FastBridge: Closing the Model-Based Realization Gap in Safety Filters on 3D Gaussian Splatting for Fast Quadrotor Flight
}

\author{Dario~Tscholl$^*$,
        Yashwanth~Nakka,
        and~Brian~Gunter%
\thanks{$^*$Corresponding author. All authors are with the Daniel Guggenheim School of Aerospace Engineering at the Georgia Institute of Technology, Atlanta, GA, 30332 USA.}
\thanks{$^\dagger$This work has been submitted to the IEEE for possible publication. Copyright may be transferred without notice, after which this version may no longer be accessible.}
}

\maketitle

\begin{figure*}[!b]
    \centering
    \includegraphics[height=4cm, width=\textwidth]{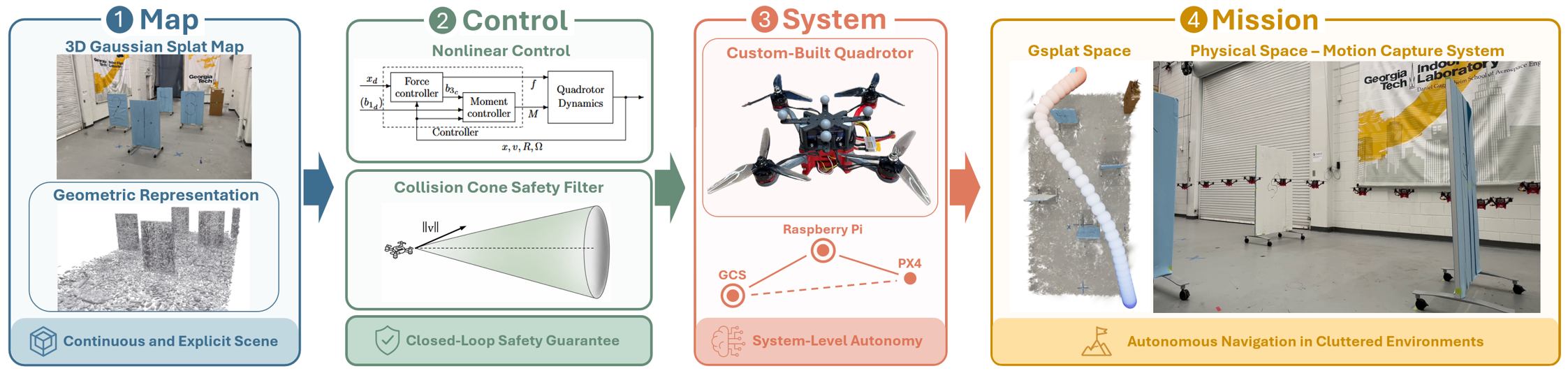}
    \caption{FastBridge architecture. FastBridge demonstrates nonlinear safety and control for a real-world quadrotor on 3DGS.}
    \label{fig:fastbridge_architecture}
\end{figure*}

\begin{abstract}
Fast quadrotor flight requires safe obstacle avoidance under tight onboard compute limits. While 3D Gaussian Splatting (3DGS) provides a continuous, geometry-aware scene representation for perception-driven navigation, existing 3DGS safety filters use reduced-order models such as single- and double-integrators that ignore actuator limits and assume commanded accelerations are realized instantaneously. Building on an analytic collision cone barrier for 3DGS, we introduce a nonlinear, actuator-aware safety filter enforced through the full quadrotor dynamics. We derive a high-relative-degree collision cone exponential CBF (ECBF) and a backup CBF whose forward-simulated backup policy provides a certified admissible fallback whenever the QP is infeasible under input constraints. Compared with a state-of-the-art 3DGS safety filter, our approach reduces trajectory jerk by 47\% and runs 2.25$\times$ faster. We demonstrate the ECBF in simulation and validate the backup CBF on hardware for real-time navigation in cluttered, perception-derived environments.$^\dagger$
\end{abstract} 
%
\section{Introduction} \label{sec_introduction}
Quadrotors can navigate cluttered environments faster than any comparable platform when flown aggressively. In time-critical missions such as search and rescue, disaster response and rapid inspection, mission value decays with time, so flying fast translates directly into mission success~\cite{hanover2024survey}. High-speed flight, however, pushes quadrotors toward their physical limits, leaving a tight onboard compute budget for perception, planning and control~\cite{foehn2022agilicious}. Fast and safe obstacle avoidance is therefore a defining challenge of agile autonomous flight.

Optimization-based methods exploit differential flatness to generate smooth, dynamically feasible trajectories, from minimum-snap polynomials~\cite{mellinger_minimum_2011} to model predictive contouring control at the limit of actuation~\cite{romero_model_2022}. Learning-based methods instead train policies mapping observations directly to commands, with reinforcement learning reaching champion-level racing~\cite{kaufmann_champion_2023}. Both families fold obstacle avoidance into a planning or training stage and offer limited runtime guarantees. Control barrier functions (CBFs) take a complementary view. Rather than replanning, CBFs are usually enforced through a quadratic program (QP) to minimally adjust the nominal control input and thus maintain safety. This lightweight filtering is ideal for compute-constrained, safety-critical flight but the CBF guarantees and conservativeness depend heavily on how the system model and the environment are represented.

Traditional navigation pipelines encode obstacles using occupancy grids, point clouds or signed-distance fields. Compared to radiance-field representations such as Neural Radiance Fields (NeRFs)~\cite{Mildenhall_nerf_2020} and 3D Gaussian Splatting (3DGS) \cite{kerbl_3d_2023}, these are discrete and lack photometric information that could support semantic reasoning. NeRFs enable vision-only quadrotor navigation~\cite{adamkiewicz_vision_2022} but are implicit and costly to query for collision checking. 3DGS instead represents geometry explicitly as Gaussian ellipsoids, enabling fast rendering and direct spatial queries that support safe, real-time navigation. Research on safety in 3DGS-derived environments is sparse, with existing works enforcing constraints either on inflated map approximations or directly on the splat ellipsoids. As an example of the former, PolyMerge builds a voxel occupancy grid, over-approximates the scene with convex polytopes and enforces avoidance through a polytope-based CBF~\cite{hong_polymerge_2026}. Although lightweight, this adds conservatism. SAFER-Splat instead operates on the scene directly, applying a distance-based CBF to the splats themselves~\cite{chen_safer-splat_2025}. Exploiting the fact that each splat can be represented as an ellipsoid, previous work derives an analytic collision cone CBF that activates proactively along the direction of motion, yielding smoother avoidance at lower computational cost than distance-based filters~\cite{tscholl_perception_2025}.

These 3DGS safety filters all enforce safety on a reduced-order model whose input is the translational velocity or acceleration. This implicitly assumes any command is realized instantaneously, ignoring actuator limits and offering no guarantee that the safe set is control invariant under them. The result is a model-based realization gap: The filter certifies the CBF condition for the rate of the barrier under the reduced-order model, while the quadrotor realizes a different rate, so forward invariance can be lost even though the commanded condition holds. Prior CBF work tracks reduced-order safe commands with sufficiently accurate low-level control \cite{molnar_model_2022} or under bounded inputs \cite{molnar_safety_2023}, treating the full-order dynamics as a bounded tracking error. We close this gap by enforcing the barrier through the full-order dynamics. Since QP feasibility under rotor limits cannot be assumed {\it a priori} during aggressive flight, we further introduce a  collision cone backup CBF that certifies a control-invariant safe subset by forward-simulating a backup policy.

This letter extends previous work~\cite{tscholl_perception_2025} by adding full nonlinear quadrotor dynamics, a realization error characterization, high-order collision cone CBF, backup CBF, actuator-aware constraints, fast-flight simulations and hardware validation. Our main contributions are summarized as follows: 
\begin{itemize}
    \item[\textit{(i)}] We characterize a model-based realization gap failure mode in reduced-order 3DGS safety filters and quantify how the unrealized barrier rate grows with flight speed. 
    \item[\textit{(ii)}] We derive a collision cone exponential control barrier function (ECBF) to close the model-based realization gap and show that the smoothness and computational advantages of the collision cone barrier are retained under the full quadrotor dynamics.
    \item[\textit{(iii)}] We develop a collision cone backup CBF whose backup policy respects the actuator limits and provably satisfies the backup CBF condition whenever its simulated trajectory is safe, providing a certified fallback when the QP is infeasible.
    \item[\textit{(iv)}] We validate the collision cone backup CBF on a real quadrotor in a reconstructed 3DGS scene.
\end{itemize}
\noindent The remainder of this letter is organized as follows. Section~\ref{sec_preliminaries} introduces the problem setup and CBF preliminaries, Section~\ref{sec_methods} develops the control stack and the collision cone safety filters, Section~\ref{sec_simulation} presents the experiments and  Section~\ref{sec_conclusion} concludes and outlines future work.

\section{PRELIMINARIES} \label{sec_preliminaries}
    %
    \subsection{Problem Formulation} \label{sec_problem}
    We consider a quadrotor flying through a cluttered environment. The goal is to track a reference trajectory that drives the vehicle from an initial state $x(0)=x_0$ to a goal state $x(T)=x_f$ while remaining collision-free. We consider a control-affine system of the form
    \begin{equation} \label{eq_affine_control_system}
        \dot{x} = f(x) + g(x)u, \qquad x(0) = x_0,
    \end{equation}
    where $x \in \mathcal{X} \subset \mathbb{R}^n$ is the state and $u \in \mathcal{U} \subset \mathbb{R}^m$ the admissible control input. The functions $f: \mathbb{R}^n \rightarrow \mathbb{R}^n$ and $g : \mathbb{R}^n \rightarrow \mathbb{R}^{n \times m}$ are continuous and locally Lipschitz.
    The quadcopter is modeled as a rigid body on $\mathrm{SE}(3)$ with state
    $(p, \sigma, v, \omega)^\top \in \mathbb{R}^3 \times \mathrm{SO}(3) \times \mathbb{R}^3 \times \mathbb{R}^3$, whose elements denote the position, attitude, linear velocity and angular velocity, respectively. The attitude is parametrized using Modified Rodrigues Parameters (MRPs) \cite{crassidis_attitude_1996}. We assume the attitude remains in a compact subset of the MRP domain that excludes the singularity at 360$^\circ$ such that the rotation matrix $R(\sigma)$ and attitude matrix $Z(\sigma)$ are smooth and bounded. The input is $u = (T, \tau_m)^\top \in \mathbb{R}^4$, where $T \in \mathbb{R}_{+}$ is the total thrust and $\tau_m \in \mathbb{R}^3$ are the body torques. The kinematic and dynamic equations are
    \begin{gather} \label{eq_mrp_dynamics}
        \dot p = v, \quad \dot \sigma = Z(\sigma) \omega, \\
        m\dot{v} = -mge_3 + TR(\sigma)^\top e_3, \quad \label{eq_mrp_rotation}
        J\dot{\omega} = -\omega \times J\omega + \tau_m.
    \end{gather}
    The environment is a scene reconstructed using 3D Gaussian Splatting \cite{kerbl_3d_2023}, in which the Gaussian splats act as obstacles. Each splat carries a mean $\mu \in \mathbb{R}^3$, a covariance $\Sigma \in \mathbb{S}^3_{++}$, an opacity and spherical-harmonic color coefficients. However, for collision avoidance, only the geometric properties $(\mu, \Sigma)$ are relevant. Unlike a traditional covariance matrix, the covariance matrix  $\Sigma = R S S^\top R^\top$ used in 3DGS describes the spatial configuration of a splat where $R \in \mathrm{SO}(3)$ is defined by a unit quaternion and $S = \mathrm{diag}(s_1, s_2, s_3)$ with $s_i \in \mathbb{R}_{++}$ is a diagonal scaling matrix. Each splat can be represented as a confidence ellipsoid $\mathcal{E} = \{x \in \mathbb{R}^3 \mid (x-\mu)^\top \Sigma^{-1} (x-\mu) \leq c^2\}$, where $c^2 = \chi^2_{3, 0.99}$ is the 99th percentile of the chi-squared distribution with three degrees of freedom. In this work, we assume a \(1\sigma\)-ellipsoid for the obstacle geometry, as this threshold most accurately represented the scene geometry in our experiments.
    In terms of safety, the vehicle is collision-free if its position remains outside every splat ellipsoid $\mathcal{E}_i$ for all $t\in[0,T]$. Safety is therefore formalized as the forward invariance of a safe set \cite{ames_control_2016}. For a continuously differentiable function $h:\mathcal{X}\to\mathbb{R}$, define the safe set as the superlevel set $\mathcal{C}=\{x\in\mathcal{X}\mid h(x)\geq 0\}$, with boundary $\partial\mathcal{C}=\{x\mid h(x)=0\}$ and interior $\mathrm{Int}(\mathcal{C})=\{x\mid h(x)>0\}$. To enforce forward invariance of $\mathcal{C}$ through feedback, we use a CBF \cite{ames_control_2019}. The function $h$ is a CBF for (\ref{eq_affine_control_system}) on $\mathcal{C}$ if there exists an extended class-$\mathcal{K}$ function $\alpha$ (continuous, strictly increasing, with $\alpha(0)=0$) such that
    \begin{equation}\label{eq_cbf_def}
        \sup_{u\in\mathcal U}\big[L_f h(x)+L_g h(x)\,u\big]\ge -\alpha\big(h(x)\big),\quad\forall\,x\in\mathcal X ,
    \end{equation}
    where $L_fh(x)$ and $L_gh(x)$ are the Lie derivatives of the barrier along the functions $f$ and $g$, respectively. Since \eqref{eq_cbf_def} is affine in $u$, it defines at each state a set of admissible inputs and any locally Lipschitz feedback derived from it renders $\mathcal C$ forward invariant. This affine structure allows CBFs to be enforced efficiently in a quadratic program (QP) that minimally modifies a reference input $\hat{u}$
    \begin{equation} \label{eq_cbf_qp}
        \begin{aligned}
            u^*= \; & \underset{u \in \mathcal{U}}{\arg \min} \; \tfrac{1}{2} 
            \| u - \hat{u} \|^2 \\
            \mathrm{s.t.} \; & L_f h(x) + L_g h(x) u \geq -\alpha\bigl(h(x)\bigr),
        \end{aligned}
    \end{equation}
    with $u^*$ being the resulting safe control input. This CBF-QP formulation is often also referred to as a safety filter and has become a standard approach for enforcing safety in robotic systems. The specific barrier $h$ used in this letter is derived in Section \ref{subsec_cc}.
    %
    \subsection{Exponential Control Barrier Function} \label{subsubsec:ecbf}
    The CBF condition in \eqref{eq_cbf_def} presumes that $h$ has relative degree one. For constraints such as a position limit on a force- or torque-controlled system, this assumption does not hold and $h$ has relative degree $r\ge 2$. Exponential CBFs (ECBFs) extend CBFs to arbitrary relative degree \cite{nguyen_exponential_2016}, with the barrier-derivative vector $\eta_b(x)=[\,h(x)\;L_fh(x) \, \cdots \, L_f^{r-1}h(x)\,]^\top$. The barrier $h$ is an ECBF if there exists $K_\alpha\in\mathbb{R}^r$ with $\sup_{u\in\mathcal U}\big[L_f^r h(x)+L_gL_f^{r-1}h(x)\,u\big]\ge -K_\alpha\,\eta_b(x)$, for all $x\in\mathrm{Int}(\mathcal C)$, where $K_\alpha$ is chosen by pole placement so that the induced barrier dynamics are Hurwitz, guaranteeing forward invariance of $\mathcal C$ whenever $h(x_0)\ge 0$~\cite{nguyen_exponential_2016, ames_control_2019}. The ECBF condition remains affine in $u$ and can therefore be directly incorporated into a QP of the same form as \eqref{eq_cbf_qp}, preserving the real-time tractability of the standard CBF approach while accommodating high-relative-degree safety constraints.
    %
    \subsection{Backup Control Barrier Functions}\label{subsubsec:backup_cbf}
    Guaranteeing CBF-QP feasibility throughout $\mathcal C$ requires control invariance under $\mathcal U$, which is computationally challenging to verify for high-dimensional nonlinear systems \cite{chen_backup_2021}. Backup CBFs \cite{chen_backup_2021,gurriet_online_2018} instead certify a control-invariant subset of $\mathcal C$ through forward simulation under a backup policy. Let $\pi:\mathcal X\to\mathcal U$ render a terminal set $\mathcal S_0\subseteq\mathcal C$ forward invariant, and let $\Phi(x,s)$ denote the flow of $\dot x=f(x)+g(x)\pi(x)$. The states whose backup trajectories remain in $\mathcal C$ over $[0,T]$ and reach $\mathcal S_0$ at $T$ form a set $\mathcal S$ with $\mathcal S_0\subseteq\mathcal S\subseteq\mathcal C$, which is the superlevel set of the backup CBF. The backup CBF is the minimum of the barrier along that trajectory and the terminal margin to $\mathcal S_0$. This requires (i) admissibility, $\pi(x)\in\mathcal U$ for all $x$, (ii) regularity, $\pi\in C^1$, and (iii) convergence to $\mathcal S_0$. Under (i)--(iii), $\pi$ is an admissible fallback whenever the QP is infeasible. Integrating the flow sensitivity $\partial\Phi/\partial x$ alongside the trajectory yields constraints affine in $u$ that have relative degree one, avoiding a high-order extension. Section~\ref{subsec_backup_cc} derives the collision cone backup CBF and verifies (i)--(iii).

\section{METHODS} \label{sec_methods}
As shown in Figure \ref{fig:fastbridge_architecture}, FastBridge uses two control layers: An established nominal tracking controller and our proposed filters, which minimally modify the nominal control input to enforce safety. For the nominal controller we use a geometric controller that tracks a smooth minimum-snap reference defined by seventh-degree polynomials (Sec.~\ref{subsec_geometric_control}). We apply dynamic feedback linearization to the full quadrotor dynamics, following \cite{cai_model_2021}, to obtain a linear input-output representation (Sec.~\ref{subsec_dfl}). We then review the collision cone barrier \cite{tscholl_perception_2025} and develop the new safety filters in increasing model fidelity (Secs.~\ref{subsec_cc}--\ref{subsec_backup_cc}).
%
    \subsection{Geometric Control} \label{subsec_geometric_control}
    A quadcopter, as described in Section \ref{sec_problem}, is differentially flat, so a smooth reference $(p_{\mathrm{ref}},\psi_{\mathrm{ref}})$ and its derivatives determine the reference thrust, attitude and angular rates. We combine this flatness construction with the geometric tracking law of \cite{lee_geometric_2010} where we use the jerk and snap to generate analytic feedforward terms rather than relying on numerical differentiation \cite{mellinger_minimum_2011}. With position and velocity errors $e_p = p - p_{\mathrm{ref}}$, $e_v = v - v_{\mathrm{ref}}$ and positive-definite gains $K_p,K_v$, the commanded inertial force $f_c = m\bigl(a_{\mathrm{ref}} + g e_3\bigr) - K_p e_p - K_v e_v,$
    determines the desired body-$z$ axis $b_3 = f_c/\lVert f_c\rVert$ (for $f_c\neq 0$) and the collective thrust $T = f_c^\top R e_3$. Differentiating $f_c$ yields the reference angular velocity and acceleration analytically. The attitude is then regulated by the geometric body-torque of~\cite{lee_geometric_2010}.
    %
    \subsection{Dynamic Feedback Linearization} \label{subsec_dfl}
    For the high-order safety filter, we again exploit the quadrotor's differential flatness to impose constraints on the translational motion directly in flat-output coordinates $y = [x,y,z,\psi(\sigma)]^\top$. Because the quadcopter is underactuated, the torques reach the position output only at higher derivatives, so static feedback does not linearize the system. Following \cite{cai_model_2021}, we add a dynamic extension on the thrust channel, giving the augmented state $\bar{x} = [p,\sigma,v,\omega,u_T,\dot u_T]^\top$ ($u_T = T/m$) and input $\bar{u} = [\ddot u_T,\tau_m]^\top$. Differentiating $y_p$ four times and $y_\psi$ twice gives the control-affine form $\bar{y} = F_{\mathrm{tot}}(\bar{x}) + M(\bar{x})\,\bar{u}$, with $M$ invertible on the operating domain and thus away from $u_T = 0$ (explicit $F_{\mathrm{tot}},M$ in~\cite{cai_model_2021}). Including yaw keeps $M$ square and avoids internal zero dynamics. The linearizing input $\bar{u} = -M^{-1}F_{\mathrm{tot}} + M^{-1}v$ with diffeomorphism $w = \xi(\bar{x}) = [p,\dot{p},\ddot{p},p^{(3)},\psi,\dot{\psi}]^\top$ reduces the closed loop to Brunovsky form $\dot{w} = Aw + Bv$, $y = Cw$, with the new input $v = [y_p^{(4)\top},\ddot{y}_\psi]^\top$.
    %
    \subsection{Collision Cone Control Barrier Function} \label{subsec_cc}
    A collision cone characterizes a set of relative velocities that would ultimately lead to collision \cite{tayal_a_2026}. As introduced in Section \ref{sec_problem}, each Gaussian splat can be formulated as an ellipsoid. Hence, if we assume the quadcopter to be a point robot moving at a constant velocity along a ray, we can solve a ray-ellipsoid intersection problem to obtain the necessary and sufficient conditions for a collision cone on 3DGS
    \begin{align}
       (v^\top A v)(r^\top A r - c^2) - (r^\top A v)^2 &\le 0, \label{eq:cone_main} \\
       r^\top A v &\ge 0. \label{eq:cone_approach}
    \end{align}
    The full derivation can be found in \cite{tscholl_perception_2025}. To enforce safety, we take the complement of the collision cone \eqref{eq:cone_main}, which yields a collision cone control barrier function
    \begin{equation}\label{eq:cc_h}
        h(r,v) = \underbrace{(v^\top A v)}_{\beta}\,
                 \underbrace{(r^\top A r - c^2)}_{\gamma}
                 -\underbrace{(r^\top A v)^2}_{\delta^2}\ \ge\ 0,
    \end{equation}
    with the safe set $\mathcal{C}_i := \{(p,v) : h_i(p,v) \geq 0 \}$ defined for the $i$-th splat. The scalars $\beta$, $\gamma$ and $\delta$ are introduced only to simplify notation. Unlike distance-based barrier functions, the collision cone barrier accounts for the direction of motion. As a result, the constraint activates proactively when the robot is moving toward an obstacle, even before the robot is close to an ellipsoid. For a double-integrator (DI) model $\dot v = a_{\mathrm{cmd}}$, the input is the commanded translational acceleration $a_{\mathrm{cmd}}\in\mathbb{R}^3$, so the collision cone barrier has relative degree one and the DI safety filter returns $a_{\mathrm{cmd}}$ directly. However this model implicitly assumes that any commanded acceleration can be realized instantaneously and without actuator limits. This assumption is not valid for a real-world quadrotor, resulting in a realization gap between the reduced-order and full-order models as mentioned in Section~\ref{sec_introduction}. On the quadrotor, the commanded acceleration is applied as the gravity-compensated force $f_{\mathrm{cmd}} = m(a_{\mathrm{cmd}} + g e_3)$, yet the acceleration the airframe actually achieves follows from the true dynamics \eqref{eq_mrp_dynamics},
    \begin{equation}
        a_{\mathrm{real}} := \dot v = -g e_3 + \tfrac{T}{m} R(\sigma)^\top e_3 ,
        \label{eq:areal}
    \end{equation}
    with the thrust $T$ and attitude $\sigma$ given by the closed loop. The realization error is then
    \begin{equation}
        e_a := a_{\mathrm{real}} - a_{\mathrm{cmd}}, \label{eq:realization_error}
    \end{equation}
    i.e., the difference between the acceleration commanded by the reduced-order filter and the acceleration the airframe realizes. Under perfect attitude-thrust tracking $T R(\sigma)^\top e_3 = f_{\mathrm{cmd}}$, so $e_a = 0$ (the ideal DI in gray, shown in Fig.~\ref{fig_realization_error}).

    \begin{proposition}[Realization gap in the barrier rate] \label{prop_realization}
        Let the collision cone barrier $h(p,v)$ be continuously differentiable, with $a_{\mathrm{cmd}}$, $a_{\mathrm{real}}$, $e_a$ as in \eqref{eq:areal}--\eqref{eq:realization_error}. Then the barrier rate realized on the airframe and the rate predicted by the reduced-order model are related by
        \[
            \dot h_{\mathrm{real}} = \dot h_{\mathrm{DI}} + \nabla_v h^\top e_a .
        \]
    \end{proposition}
    \begin{proof}
    The barrier depends on the state only through position (via $r=\mu-p$) and velocity $v$, so along any trajectory $ \dot h = \nabla_p h^\top \dot p + \nabla_v h^\top \dot v = \nabla_p h^\top v + \nabla_v h^\top \dot v$, 
    where $\dot p = v$ holds for both models. Evaluating with the modeled input $\dot v = a_{\mathrm{cmd}}$ gives $\dot h_{\mathrm{cmd}} = \nabla_p h^\top v + \nabla_v h^\top a_{\mathrm{cmd}}$, and with the achieved input $\dot v = a_{\mathrm{real}}$ gives $\dot h_{\mathrm{real}} = \nabla_p h^\top v + \nabla_v h^\top a_{\mathrm{real}}$. Subtracting cancels the common drift $\nabla_p h^\top v$ and leaves 
    $\dot h_{\mathrm{real}} - \dot h_{\mathrm{cmd}}
       = \nabla_v h^\top(a_{\mathrm{real}}-a_{\mathrm{cmd}})
       = \nabla_v h^\top e_a$.
    \end{proof}
    The DI filter selects $a_{\mathrm{cmd}}$ so that the relative-degree-one condition $\dot h_{\mathrm{cmd}} + \alpha(h)\ge 0$ holds in the reduced model. By
    Proposition~\ref{prop_realization}, the condition actually realized on the airframe is $\dot h_{\mathrm{real}} + \alpha(h) = \dot h_{\mathrm{cmd}} + \alpha(h) + \nabla_v h^\top e_a$. Whenever the unmodeled term $\nabla_v h^\top e_a$ is sufficiently negative, the realized condition is violated even though the commanded one is satisfied and forward invariance of the safe set is lost.
    %
    \subsection{Exponential Collision Cone CBF} \label{subsec_nonlinear_cc}
    For the full nonlinear model, the true input $(T,\tau_m)$ first appears in the flat output at the snap $p^{(4)}$, as shown by the dynamic feedback linearization (DFL) in Sec.~\ref{subsec_dfl}, so the collision cone barrier has relative degree three. We therefore enforce it with a third-order exponential CBF (ECBF) with constraint functions $\psi_0=h$ and $\psi_i=\dot\psi_{i-1}+\alpha_i(\psi_{i-1})$. Rather than applying the DFL to the closed-loop system, we embed it in the ECBF constraint, which lets us formulate the QP \eqref{eq_nonlinear_cbf_qp} in the true input. Considering linear class-$\mathcal{K}$ functions $\alpha_1=\alpha_2=\alpha_3=\lambda$, where $\lambda>0$ (a triple pole at $-\lambda$), gives the high-order CBF condition $L_f^3 h + L_g L_f^2 h\,u + \alpha_3(\psi_2) \ge 0$. Collecting the snap yields $\dot\psi_2 = w^\top p^{(4)} + \varphi(r,v,a,j)$, with control direction $w = 2A(\gamma v - \delta r)$ and a known drift $\varphi$ in the lower derivatives. The snap is itself affine in the input,
    \begin{equation}\label{eq:snap_affine}
      p^{(4)} = F_p(\bar x, u_T) + G_T(\bar x, u_T)\,\ddot u_T
              + G_\tau(\bar x, u_T)\,\tau_m,
    \end{equation}
    %
    where $F_p$, $G_T$, $G_\tau$ are the drift and input maps of the snap. The body torque $\tau_m$ enters directly while $T$ enters through $u_T=T/m$ and the backward-difference approximations $\dot u_T\approx(T-T^-)/(m\Delta t)$ and $\ddot u_T\approx(T-2T^-+T^{--})/(m\Delta t^2)$, with $\Delta t$ the zero-order hold period and $T^-,T^{--}$ the two previously applied thrusts. Substituting \eqref{eq:snap_affine} and the triple-pole gains and evaluating the bilinear term $u_T\tau_m$ at $T^-$, gives a constraint affine in $(T,\tau_m)$
    \begin{equation} \label{eq:ecbf_constraint}
      w^\top G_\tau \tau_m \ge
      -w^\top G_T \ddot u_T - w^\top F_p - \varphi - 3\lambda \ddot h- 3\lambda^2 \dot h-\lambda^3 h.
    \end{equation}
    Therefore, we can formulate the safety filter as an ECBF-QP subject to the hard safety constraint \eqref{eq:ecbf_constraint} and a rotor polytope to enforce actuator constraints
    \begin{equation} \label{eq_nonlinear_cbf_qp}
        \begin{aligned}
            u^*= \; & \underset{u}{\arg \min} \; \tfrac{1}{2}
            \| u - \hat{u} \|^2 \\
            \mathrm{s.t.} \ \
            & \text{\eqref{eq:ecbf_constraint}} , \\
            \ \ & f_{\min} \le G^{-1}u \le f_{\max}, \\
        \end{aligned}
    \end{equation}
    where $u = (T, \tau_m)$ is the nonlinear control input and $G$ is a rotor allocation matrix.
    
    \subsection{Collision Cone Backup CBF} \label{subsec_backup_cc}
    The ECBF-QP~\eqref{eq_nonlinear_cbf_qp} is feasible only if $\mathcal{C}$ is control invariant under~$\mathcal{U}$ (Sec.~\ref{subsubsec:backup_cbf}). During fast, aggressive flight this cannot be guaranteed {\it a priori}, which motivates the collision cone backup CBF. 
    %
    As $v\to0$, the barrier in \eqref{eq:cc_h} vanishes together with its gradient so its rows along the braked part of the backup trajectory would degenerate to $0\ge0$. For the backup CBF, we therefore replace the running barrier with the look-ahead barrier 
    \begin{equation} \label{eq:lookahead} 
        \widehat h_i(p,v):=\min_{t\in[0,\ell]}\gamma_i(p+tv)=\beta(t^\star)^2-2\delta t^\star+\gamma,
    \end{equation}
    where $\ell:=T/N$ is the rollout grid spacing and $t^\star:=\min\{\ell,\max\{0,\delta/\beta\}\}$ is the time of closest approach in the $A_i$-metric, clipped to $[0,\ell]$ with $t^\star=0$ for $v=0$. Thus, $\widehat h_i\ge0$ if and only if the segment $\{p+tv:t\in[0,\ell]\}$ avoids the interior of $\mathcal E_i$ and since the window contains $t=0$, $\widehat h_i\le\gamma_i(p)$, so the safe set $\mathcal C:=\bigcap_i\{x:\widehat h_i(p,v)\ge0\}$ contains only states outside the splats. When $\delta/\beta\in[0,\ell]$, $\widehat h_i=h_i/\beta$, so the two barriers agree in sign. For $\delta\le0$, $\widehat h_i$ reduces to the clearance $\gamma_i(p)$, replacing the approach gate \eqref{eq:cone_approach} and as $v\to0$ it converges to $\gamma_i(p)$, which later serves as the terminal barrier. For $v\neq0$, the envelope theorem gives the continuous gradients $\nabla_p\widehat h_i=-2A_iz_i$ and $\nabla_v\widehat h_i=-2t^\star A_iz_i$ with $z_i:=r_i-t^\star v$. At $v=0$, $\widehat h_i$ is only directionally differentiable, which Proposition~\ref{prop:feasibility} accounts for.
    %
        \subsubsection{Backup Policy and Implicit Control-Invariant Set}
        The backup policy $\pi_b:\mathbb{R}^{12}\!\to\!\mathcal{U}$ is a saturated hover regulator. Its raw commands are a gravity feedforward with vertical-velocity damping and a PD-plus-braking torque,
        \begin{equation} \label{eq:pi-unsat}
        \begin{aligned}
          \widetilde{T} &= mg - m\,k_{v_z}v_z, \\
          \widetilde{\tau}_m &= -K_\sigma\sigma - K_\omega\omega + k_{v_h}\,S\,R(\sigma)\,v,
          \end{aligned}
        \end{equation}
        with gains $k_{vz},k_{vh},K_\sigma,K_\omega>0$ and $S:=-\hat e_3$, where $\hat{(\cdot)}$ is the skew-symmetric operator, so that $S\,R(\sigma)v$ rotates the horizontal body velocity by $-\tfrac{\pi}{2}$ to produce a braking tilt. Each channel is then mapped smoothly into a box centered at hover,
        \begin{equation} \label{eq:pi-sat}
        \begin{aligned}
          T_{\mathrm{sat}} &= mg + c_T\tanh\!\Big(\tfrac{\kappa(\widetilde{T}-mg)}{c_T}\Big), \\
          \tau^{\mathrm{sat}}_{m,i} &= \tau_{i,c}\tanh\!\Big(\tfrac{\kappa\,\widetilde{\tau}_{m,i}}{\tau_{i,c}}\Big),
          \end{aligned}
        \end{equation}
        giving $\pi_b=(T_{\mathrm{sat}},\tau^{\mathrm{sat}}_{m,x},\tau^{\mathrm{sat}}_{m,y},\tau^{\mathrm{sat}}_{m,z})^\top$, where $\kappa>0$ sets the sharpness of the saturation.
        The caps $c_T,\tau_{i,c}$ are sized to the per-rotor thrust margin about hover. With the rotor allocation $u=Gf$ and bounds $f_{\min},f_{\max}$ of~\eqref{eq_nonlinear_cbf_qp}, the hover allocation is $f_{\mathrm{hov}}=\tfrac{mg}{4}\mathbf 1_4$, $u_{\mathrm{hov}}=Gf_{\mathrm{hov}}=(mg,0,0,0)^\top$, with per-rotor thrust margin $\eta_i:=\min(f_{\max,i}-f_{\mathrm{hov},i},\,f_{\mathrm{hov},i}-f_{\min,i})>0$. For a design weight $d\in\mathbb{R}^4_+$ ($d=(2.0,0.06,0.07,0.02)$ in our tests), set
        \begin{equation} \label{eq:s-def}
          s:=\eta_{\mathrm{frac}}\min_i\frac{\eta_i}{(|G^{-1}|\, d)_i}, \quad (c_T,c_x,c_y,c_z)^\top:=s\,d,
        \end{equation}
        safety fraction $\eta_{\rm frac}\in(0,1)$ ($0.95$ in our tests) and $\tau_{i,c}=c_i$, so that $T_{\rm sat}\in(mg-c_T,\,mg+c_T)$.
        
        \begin{proposition}[Admissibility] \label{prop:admissibility}
            For every $x\in\mathbb{R}^{12}$, $\pi_b(x)\in\mathcal{U}$.
        \end{proposition}
        \begin{proof}
            Let $\Delta:=\pi_b-u_{\mathrm{hov}}$. By \eqref{eq:pi-sat}, $|\Delta_T|\le c_T$ and $|\Delta_{\tau_i}|\le\tau_{i,c}=c_i$. The per-rotor thrust is $f=G^{-1}\pi_b=f_{\mathrm{hov}}+G^{-1}\Delta$, so $|f_i-f_{\mathrm{hov},i}|\le\sum_k|G^{-1}_{ik}|\,c_k=(|G^{-1}|sd)_i$. Substituting $s$ gives $|f_i-f_{\mathrm{hov},i}|\le\eta_{\mathrm{frac}}\,\frac{\eta_{i^\star}}{(|G^{-1}|d)_{i^\star}}\,(|G^{-1}|d)_i\le\eta_{\mathrm{frac}}\,\eta_i$, where $i^\star$ is the rotor attaining the minimum in \eqref{eq:s-def}. With the thrust margin definition this gives $f_{\min,i}\le f_i\le f_{\max,i}$, i.e.\ $\pi_b(x)\in\mathcal{U}$, with strict slack $1-\eta_{\mathrm{frac}}$.
        \end{proof}
        This construction satisfies the three requirements of Sec.~\ref{subsubsec:backup_cbf}. \emph{(i) Admissibility} follows from Prop.~\ref{prop:admissibility}. \emph{(ii) Regularity:} Since $\tanh$ and $R(\sigma)$ are smooth away from the MRP singularity, $\pi_b\in C^\infty$ and is Lipschitz on the compact attitude domain of Section~\ref{sec_problem}, so the flow $\Phi$ and sensitivity $Q$ in~\eqref{eq:variational} exist and are unique on $[0,T]$. \emph{(iii) Convergence:} the hover equilibria $\{(p,0,0,0)\}$ of $f_{\mathrm{cl}}$ are locally exponentially stable on the $(\sigma,v,\omega)$ block, so the backup trajectory settles to a stationary hover. We therefore take the terminal set $\mathcal S_0$ as a forward-invariant neighborhood of collision-free hover states and the safe set as $\mathcal C:=\bigcap_i\{x:\widehat h_i(p,v)\ge0\}$. Following \cite{gurriet_online_2018}, the resulting set $\mathcal S$ of Sec.~\ref{subsubsec:backup_cbf} is a control-invariant subset of $\mathcal C$ that is represented implicitly by forward integration rather than computed offline.
        %
        \subsubsection{Backup Trajectory and Sensitivity Propagation}
        At each control instant we forward-integrate, from the current state ${x}$, the joint flow and its state sensitivity $Q(s)=\partial\Phi({x},s)/\partial x$,
        \begin{equation}   \label{eq:variational}
          \dot{\Phi}=f_{\mathrm{cl}}(\Phi), \qquad \dot{Q}=J_{\mathrm{cl}}(\Phi)\,Q,
        \end{equation}
        with $\Phi(0)=x$, $Q(0)=I_{12}$, and the analytically computed closed-loop Jacobian $J_{\mathrm{cl}}=\partial f_{\mathrm{cl}}/\partial x$. We integrate~\eqref{eq:variational} with RK4 and sample it at the $N{+}1$ nodes $s_k=kT/N$, yielding $\Phi_k:=\Phi(x,s_k)$ and $Q_k:=Q(s_k)$ for $k=0,\dots,N$.

        \subsubsection{Safety Filter with Certified Fallback}
        At every control instant, we predict each running barrier a fixed time $s_k$ ahead, $\eta_{i,k}(x):=\widehat h_i(\Phi_k)$, and require it to stay nonnegative. Since $s_k$ is fixed while the state evolves under $\dot x=f(x)+g(x)u$, the chain rule gives $\dot\eta_{i,k}=\nabla\widehat h_i(\Phi_k)^\top Q_k\big(f(x)+g(x)u\big)$ and imposing $\dot\eta_{i,k}+\alpha_R\eta_{i,k}\ge0$ with $\alpha_R>0$ yields the running constraint, affine in $u$,
        \begin{equation} \label{eq:running}
            \nabla\widehat h_i(\Phi_k)^\top Q_k\big(f(x)+g(x)u\big)+\alpha_R\,\widehat h_i(\Phi_k)\ge0,
        \end{equation}
        for every splat $i$ and node $k=1,\dots,N$. The $k=0$ row is omitted because $Q_0=I_{12}$ makes it independent of the torques. For $k\ge1$, $Q_k$ carries the input's effect to the future state, which yields the relative-degree-one constraints of Sec.~\ref{subsubsec:backup_cbf} \cite{chen_backup_2021}. The terminal barrier at $s=T$ is treated in the same way, with $\alpha_T>0$,
        \begin{equation}\label{eq:terminal}
            \nabla\gamma_i(\Phi_T)^\top Q_T\big(f(x)+g(x)u\big)+\alpha_T\,\gamma_i(\Phi_T)\ge0,
        \end{equation}
        which ensures that the hover state reached by the backup is collision-free. Together with the rotor limits, this gives the collision cone backup CBF-QP
        \begin{equation}\label{eq:cbfqp}
            \begin{aligned}
                u^*=\;&\underset{u}{\arg\min}\;\tfrac12\|u-\hat u\|^2\\
                \mathrm{s.t.}\ \ &\eqref{eq:running},\ \eqref{eq:terminal},\\
                &f_{\min}\le G^{-1}u\le f_{\max}.
            \end{aligned}
        \end{equation}
        The key property of this construction, made precise next, is that $\pi_b$ itself satisfies the backup CBF condition whenever its simulated trajectory is safe and settles with margin.
        %
        \begin{proposition}[Backup policy certificate] \label{prop:feasibility}
        Let $\Phi_s := \Phi(x,s)$ denote the $C^1$ backup flow of $\dot x=f_{\mathrm{cl}}(x)$ and $h_b(x):=\min\big\{\min_{i\in\mathcal I,\,s\in[0,T]}\widehat h_i(\Phi_s),\ \min_{i\in\mathcal I}\gamma_i(\Phi_T)\big\}$ be the backup CBF. Suppose $h_b(x)\ge0$, and for every obstacle $i\in\mathcal I$, the terminal residuals of the running barrier $\rho_i^{R}:=D^+\widehat h_i\big(\Phi_T;\,f_{\mathrm{cl}}(\Phi_T)\big)+\alpha_R\,\widehat h_i(\Phi_T)$ and the terminal barrier $\rho_i^{T}:=\nabla\gamma_i(\Phi_T)^\top f_{\mathrm{cl}}(\Phi_T)+\alpha_T\,\gamma_i(\Phi_T)$ are nonnegative, where $D^+\widehat h_i(x;d)$ denotes the one-sided directional derivative of $\widehat h_i$ at $x$ along $d$, which equals $\nabla\widehat h_i(x)^\top d$ whenever $v\ne0$. Then $u=\pi_b(x)$ satisfies the backup CBF condition $\dot h_b+\alpha h_b\ge0$, where $\alpha:=\max\{\alpha_R,\alpha_T\}$.
        \end{proposition}
        %
        \begin{proof} 
        Define $q_i(x,t):=\gamma_i(p+tv)$, so that $\widehat h_i(x)=\min_{t\in[0,\ell]}q_i(x,t)$. Thus, $h_b$ is the minimum of the smooth branches $q_i(\Phi_s,t)$ and $\gamma_i(\Phi_T)$ over a compact index set. Their gradients are jointly continuous, so Danskin's theorem gives $\dot h_b$ as the smallest rate among the active branches, i.e., those attaining $h_b(x)$. With $Q_s:=D_x\Phi(x,s)$, the flow identity $Q_sf_{\mathrm{cl}}(x)=f_{\mathrm{cl}}(\Phi_s)$ \cite{chen_backup_2021} gives the rate of an active running branch as $\nabla_x q_{i^\star}(\Phi_{s^\star},t^\star)^\top Q_{s^\star}f_{\mathrm{cl}}(x)=\left.\frac{d}{ds}q_{i^\star}(\Phi(x,s),t^\star)\right|_{s=s^\star}$, with $t^\star$ held fixed. If $s^\star<T$, this rate is nonnegative because $s^\star$ minimizes $q_{i^\star}(\Phi_s,t^\star)$ over $[0,T]$, so it is zero at an interior minimizer and nonnegative at $s^\star=0$. If $s^\star=T$, then $t^\star$ is an inner minimizer and $\widehat h_{i^\star}(\Phi_T)=h_b(x)$. Applying Danskin's theorem to the inner minimum therefore bounds the active branch's rate below by $D^+\widehat h_{i^\star}(\Phi_T;f_{\mathrm{cl}}(\Phi_T))=\rho_{i^\star}^R-\alpha_R h_b(x)\ge-\alpha_R h_b(x)$. An active terminal branch similarly has rate $\nabla\gamma_{i^\star}(\Phi_T)^\top f_{\mathrm{cl}}(\Phi_T)=\rho_{i^\star}^T-\alpha_T h_b(x)\ge-\alpha_T h_b(x)$. Taking the minimum over all active rates and using $h_b(x)\ge0$ yields $\dot h_b(x)\ge\min\{0,-\alpha_R h_b(x),-\alpha_T h_b(x)\}=-\alpha h_b(x)$, proving the claim. 
        \end{proof}
        Since $\gamma_i$ depends only on position, the residual reduces to $\rho_i^T=\alpha_T\gamma_i(p_T)+\nabla_p\gamma_i(p_T)^\top v_T=\alpha_T\gamma_i(p_T)+\mathcal{O}(\|v_T\|)$, where $p_T:=p_{\Phi_T}$ and $v_T:=v_{\Phi_T}$. The running residual additionally depends on the terminal acceleration $a_T$ and satisfies $\rho_i^R=\alpha_R\widehat h_i(\Phi_T)+\mathcal{O}(\|v_T\|+\ell\|a_T\|)$ near hover. For $\ell=T/N$ and positive gains, exponential convergence under the smooth backup dynamics makes both residuals nonnegative for sufficiently large $T$, provided the limiting hover position has strictly positive clearance, $\gamma_i(p_\infty)>0$ for every $i\in\mathcal I$. For any state $x$ satisfying $h_b(x) \geq 0$, the pointwise constraints \eqref{eq:running}, \eqref{eq:terminal} define an inner approximation of the set of controls satisfying $\dot h_b+\alpha h_b\ge0$ when enforced over the continuous rollout for all active look-ahead branches. Because of that, the sampled implementation in \eqref{eq:cbfqp} may become infeasible even when the condition itself is feasible. In that case, the filter applies $u=\pi_b(x)$, which is admissible (Prop.~\ref{prop:admissibility}) and satisfies the backup CBF condition whenever the premise of Proposition~\ref{prop:feasibility} holds.
%
\section{EXPERIMENTS \& RESULTS} \label{sec_simulation}
In this letter we present three experiments that all use a custom real-world 3DGS scene trained with Nerfstudio's Splatfacto \cite{nerfstudio}. For numerical stability, we use a scale regularizer and scale adjustments as in \cite{tscholl_perception_2025} during training. Each trajectory follows a straight-line reference between the start and goal points, parameterized by a seventh-degree polynomial. Each safety-filter QP also includes a control Lyapunov function (CLF) on the position and velocity errors with a slack $\xi\ge0$ that relaxes the CLF constraint whenever the CBF and CLF compete \cite{ames_control_2016}.
    %
    \subsection{Reduced-Order vs Full-Order Model Comparison} \label{subsec_linear_nonlinear}
    This experiment shows why closing the realization gap between the full nonlinear dynamics and a reduced-order double integrator (DI) model matters.
    We sweep the maximum reference velocity and report the peak instantaneous realization error for each trajectory. The peak realization error is $\lVert e_a\rVert_\infty:=\sup_{t\in[0,T_f]}\bigl\lVert a_{\mathrm{real}}(t)-a_{\mathrm{cmd}}(t)\bigr\rVert_2$, the Euclidean norm over the three axes maximized over the trajectory.
    \begin{figure}[t]
        \includegraphics[width=\columnwidth] {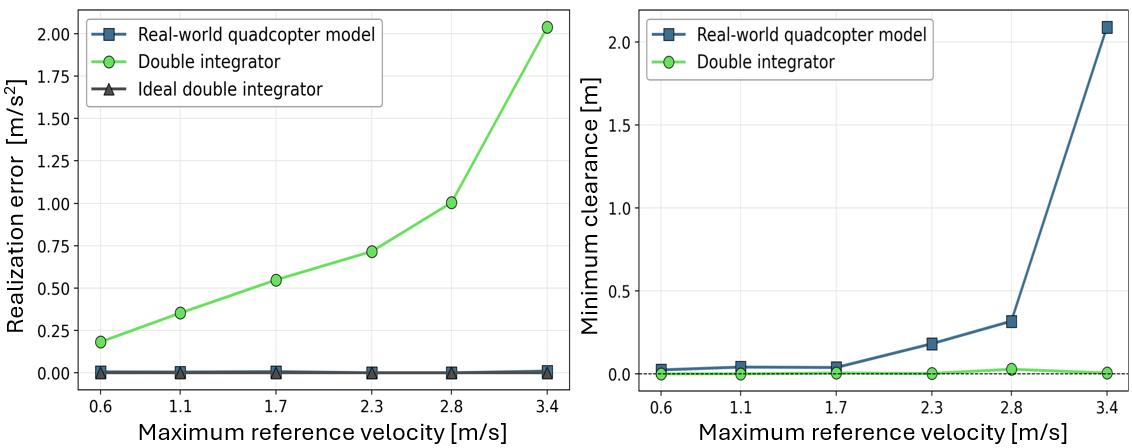} 
        \caption{Realization error (left) and minimum clearance (right) versus maximum reference velocity. The DI realization error (green) grows rapidly with speed, while the ECBF (blue) and the ideal DI (gray) are identically zero}. The DI retains almost no clearance margin, whereas the ECBF grows more conservative with speed.
        \label{fig_realization_error}
    \end{figure}
    Figure \ref{fig_realization_error} compares three formulations of the collision cone CBF. The reduced-order DI filter (green) commands an acceleration fed to the nonlinear quadrotor as $f_{\mathrm{cmd}}=m(a_{\mathrm{cmd}}+ge_3)$, so its realization error captures the mismatch between the commanded and achieved acceleration on the real quadcopter. The nonlinear ECBF (blue) enforces safety through the true dynamics, while the ideal DI (gray) runs the linear CBF on a pure double integrator and is therefore zero by construction.

    The left image shows that the DI realization error is small below \SI{1}{\meter\per\second}, where the cheaper DI may even be preferable, but grows rapidly with speed, while the ECBF error remains zero. The right image shows the minimum clearance $\gamma_{\min}=\min(r^\top Ar-c^2)$ over each trajectory. Both filters remain safe ($\gamma_{\min}\ge0$), but the DI stays close to the splat boundary, whereas the ECBF grows more conservative with speed. A growing realization error combined with almost no margin leaves the DI no room for robustness at high speeds.
    %
    \subsection{Comparison to the State-of-the-Art} \label{sec_sota_ours}
    \begin{figure*}[t!] 
        \includegraphics[width=\textwidth] {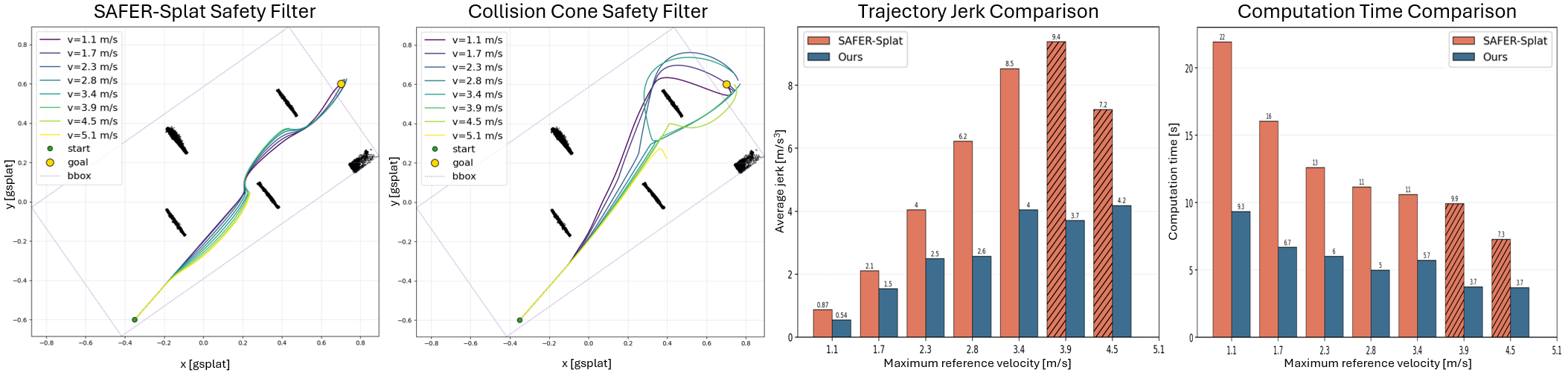}
        \caption{SAFER-Splat versus the collision cone ECBF. On the left, $xy$-trajectories for increasing reference speed. On the right, trajectory jerk and computation time. The ECBF produces smoother trajectories at a fraction of the computational cost. Where both filters reach the goal (up to $3.4\,\mathrm{m\,s^{-1}}$), the ECBF paths are on average 15\% longer and arrive 10\% later because they detour around the splat cluster at the goal. At higher speeds, only the ECBF completes the trajectory.}
        \label{fig_safer_vs_ours}
    \end{figure*}
    We compare the ECBF filter with SAFER-Splat \cite{chen_safer-splat_2025} because both enforce safety directly on the splats rather than on an over-approximation and we test whether the advantages of the collision cone barrier \cite{tscholl_perception_2025} are retained under the full system dynamics. For a fair comparison, both filters share the same start and goal points, reference trajectory, translational gains ($K_p=3$, $K_v=2$), control-loop frequency and GPU implementation. Both filters also share a distance gate that keeps only the splats within a ball around the vehicle. Figure \ref{fig_safer_vs_ours} sweeps the maximum reference speed from \SI{1.1}{\meter\per\second} to \SI{5.1}{\meter\per\second}. The first two images show the resulting $xy$-trajectories under SAFER-Splat and the ECBF whereas the last two compare their trajectory jerk and computation time. Across the sweep, the ECBF reduces the average jerk by $47\%$ and the maximum jerk by $46\%$ and computes $2.25\times$ faster, because it encodes safety as a closed-form constraint rather than solving a distance program at each step. SAFER-Splat's achieved velocity saturates near \SI{3.5}{\meter\per\second}, since a distance-based filter cannot approach obstacles arbitrarily fast without risking penetration, whereas the ECBF reaches the \SI{4.5}{\meter\per\second} reference velocity while closing the model-based realization gap. Since the conditions are held fixed across runs, both filters eventually fail to reach the goal within the allotted time, i.e., beyond a reference velocity of approximately \SI{5}{\meter \per \second} (earlier path-incompletion is marked by the shaded columns), but neither compromise safety. The experiment also shows smaller ECBF deviations from the straight-line path at higher speeds because the pole $\lambda$ sets the turning rate, so the heading change before the obstacle scales inversely with speed. How early the deviation begins and how far it extends is set by the splat-filter radius, since the filter reacts only to splats within that radius.
    %
    \subsection{Backup CBF Validation} \label{sec_nonlinear_cc_cbf}
    Since the ECBF cannot guarantee QP feasibility under actuator limits {\it a priori}, we demonstrate the collision cone backup CBF on hardware. Splatfacto reconstructs the scene in a normalized frame that preserves scene geometry but lacks metric scale. We recover metric scale and align the reconstruction with the physical-world frame by estimating a similarity transform. We retain all 87{,}356 splats in the cropped scene, each contributing a running constraint at every node and a terminal constraint at the horizon. Rows that cannot bind for any input in the rotor polytope are removed exactly by evaluating their extremes over the polytope in closed form. The QP is solved on a working set seeded with the rows closest to binding and re-solved until every remaining row is verified at the solution, which returns the same $u^*$ as a solve over all rows.
    
    The platform is a custom \SI{0.459}{\kilo\gram} quadrotor with a maximum per-rotor thrust of \SI{3.05}{\newton}, a Kakute H7 Mini flight controller running PX4 v1.17 and a Raspberry Pi Zero 2W companion computer connected over uXRCE-DDS. A ground laptop runs the nominal controller and the filter \eqref{eq:cbfqp} at 20\,Hz on CPU, evaluated on the state predicted forward over the expected delay from state sample to actuation (a fixed \SI{60}{\milli\second} for computation and downlink plus the measured age of the odometry sample, about \SI{80}{\milli\second} in total) and streams the filtered wrench $u^*=(T^*,\tau_m^*)$ to the Pi over Wi-Fi. The Pi applies $T^*$ and tracks the attitude target of $u^*$ with a local 85\,Hz attitude loop to bridge the command delay, sending thrust and torque setpoints to PX4 that bypass its attitude and rate controllers. The backup policy $\pi_b$ is applied with the torque trim and hover-thrust offset measured before engagement as feedforward so that its physical hover matches the closed loop $f_{\mathrm{cl}}$ of the rollout. A Vicon motion capture system provides position and velocity. Since Vicon does not provide acceleration or jerk, we use the acceleration and jerk calculated from the rigid-body model \eqref{eq_mrp_dynamics}, \eqref{eq_mrp_rotation} at the commanded thrust. The reference yaw faces the goal and is held constant.
    
    Figure~\ref{fig:hardware_results} shows four consecutive flights rendered in the reconstructed 3DGS scene with the robot extent shown as a sphere, together with the filter's hardware metrics. With all 87{,}356 splats in the QP, the filter met its 50~ms budget at 20~Hz on 99.8\% of the ticks and all four flights completed with positive measured clearance. These experiments verified Proposition~\ref{prop:feasibility} for the model. The backup policy $\pi_b$ satisfied $\dot h_b+\alpha h_b\ge0$ at every recorded state where the proposition's premise held.
    Two effects surfaced during the flights which are due to deviations from the model. First, the model assumes that the filtered wrench is applied immediately and continuously with ideal thrust and torque response, whereas the hardware introduces an 80~ms command delay, thrust and motor-lag mismatch and local attitude corrections. This mismatch caused the state to leave the certified set $\mathcal S$ on 2.8\% of the ticks as well as the single infeasible tick. Leaving $\mathcal S$ here means that the predicted backup trajectory lost its certificate, not that the vehicle left the safe set. The vehicle stayed clear throughout with the predicted excursions absorbed by the splat inflation radius $r_m$. Modeling these effects as a bounded disturbance through a reachable tube or an input-to-state safety bound on $r_m$ could extend the guarantee to the physical vehicle. Second, the sampled rows~\eqref{eq:running} only approximate the continuous condition, so $\pi_b$ violated some sampled running constraints. All such rows belonged to interior nodes away from each splat's minimum along the trajectory, where Proposition~\ref{prop:feasibility} imposes no condition, and the QP remained feasible at every recorded state in $\mathcal S$. Enforcing each running constraint at that minimizer rather than at fixed nodes would close this gap and make $\pi_b$ feasible for~\eqref{eq:cbfqp} whenever the premise of Proposition~\ref{prop:feasibility} holds.
    \begin{figure}[!t]
        \centering
        \includegraphics[width=\columnwidth]{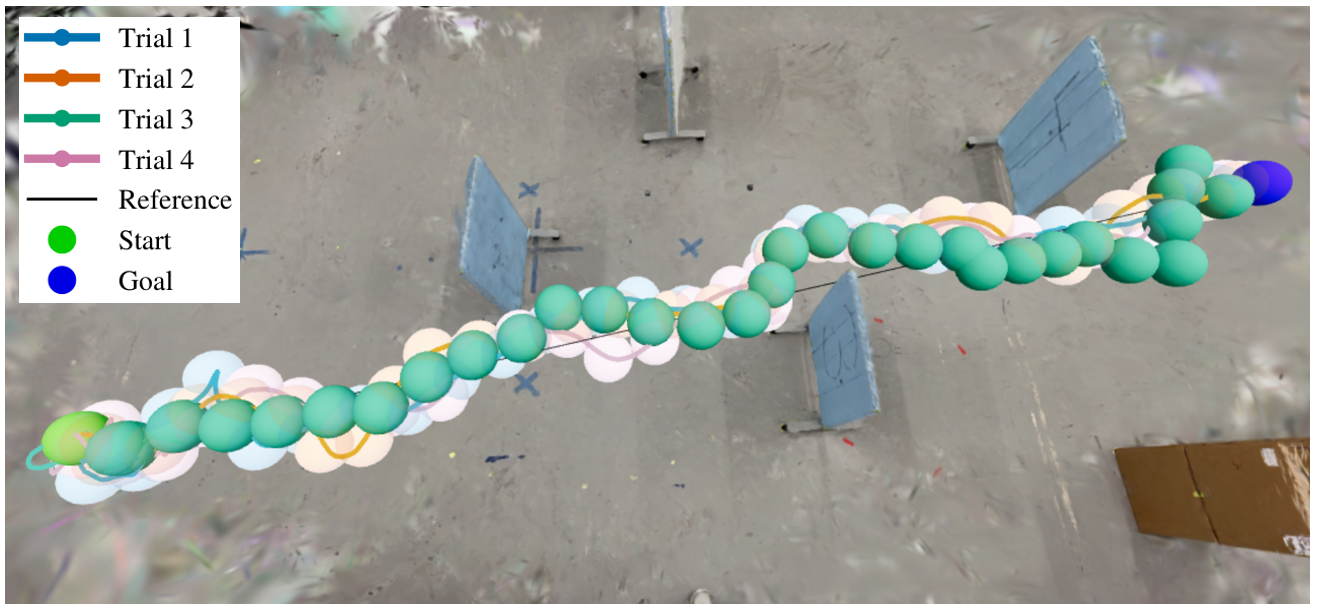}
        
        \vspace{1ex}
        
        \renewcommand{\arraystretch}{1.0}
        \begin{tabular*}{\columnwidth}{@{\hspace{0.5em}}l@{\extracolsep{\fill}}rr@{\hspace{0.5em}}}
        \toprule
            \textbf{Metric}     & \textbf{Mean}         & \textbf{Worst} \\
            \midrule
            Loop mean           & 34.3~ms               & 34.5~ms \\
            Loop p99            & 42.3~ms               & 45.0~ms \\
            Loop max            & 57.1~ms               & 72.1~ms \\
            Loop overruns       & 7 / 4092 (0.17\%)     & 3 / 993 \\
            \midrule
            Feasible ticks      & 4091 / 4092 (99.98\%)   & 992 / 993 \\
            Solver failures     & 0                     & 0 \\
            \midrule
            Cmd age p50/p95 & 81.5 / 105.4~ms     & 81.5 / 105.8~ms \\
            Cmd age max     & 132~ms                & 153~ms \\
            \midrule
            Min clearance       & 0.189~m               & 0.167~m  \\              
            Flight time         & 36.2~s                & 36.9~s \\
            \bottomrule 
        \end{tabular*}
        \caption{Hardware validation of the collision cone backup CBF. The image shows four consecutive collision-free flights in the reconstructed 3DGS scene with the robot extent shown as a sphere. The table reports the filter's timing and feasibility metrics over all four flights.}
        \label{fig:hardware_results}
    \end{figure}
%
\section{CONCLUSIONS} \label{sec_conclusion}
This letter characterized the model-based realization gap in reduced-order 3DGS safety filters and closed it by extending the collision cone barrier to the full nonlinear quadrotor dynamics through an exponential CBF. We also derived a collision cone backup CBF whose backup policy respects the actuator limits and is certified as a fallback whenever its simulated trajectory is safe. In simulation, the ECBF retained the advantages of the collision cone barrier, reducing trajectory jerk by 47\% and computing 2.25$\times$ faster than a state-of-the-art 3DGS filter and the backup CBF was validated on hardware. Future work will combine onboard computation of the backup CBF with a speed-preserving cone-escape backup policy for sustained fast flight and extend the forward-invariance guarantee to the physical system.
%

%
%
\bibliographystyle{IEEEtran}
\bibliography{root}

\end{document}